\documentclass[letterpaper]{article} 
\usepackage{aaai25}  
\usepackage{times}  
\usepackage{helvet}  
\usepackage{courier}  
\usepackage[hyphens]{url}  
\usepackage{graphicx} 
\usepackage{natbib}  
\usepackage{caption} 
\usepackage{algorithm}
\usepackage{algorithmic}
\usepackage{makecell}
\usepackage{amsmath}
\usepackage{amsfonts}

\usepackage{newfloat}
\usepackage{listings}
\DeclareCaptionStyle{ruled}{labelfont=normalfont,labelsep=colon,strut=off} 
\floatstyle{ruled}
\newfloat{listing}{tb}{lst}{}
\floatname{listing}{Listing}

\def\pAp@k{\texttt{pAp@k}}

\def\bw{\mathbf{w}}

\usepackage{multirow}
\usepackage{xcolor}
\definecolor{llgray}{HTML}{D7D7D7} 
\newcommand{\ie}{\emph{i.e.}}
\newcommand{\eg}{\emph{e.g.}}

\newtheorem{theorem}{Theorem}
\nocopyright

\title{Meta-Learning with Heterogeneous Tasks}
\author {
    Zhaofeng Si\textsuperscript{\rm 1},
    Shu Hu\textsuperscript{\rm 2},
    Kaiyi Ji\textsuperscript{\rm 1},
    Siwei Lyu\textsuperscript{\rm 1}
}
\affiliations {
    \textsuperscript{\rm 1}University at Buffalo, State University of New York\\
    \textsuperscript{\rm 2}Purdue University\\
    zhaofeng@buffalo.edu, hu968@purdue.edu, kaiyiji@buffalo.edu, siweilyu@buffalo.edu 
}

\begin{document}

\maketitle

\begin{abstract}
Meta-learning is a general approach to equip machine learning models with the ability to handle few-shot scenarios when dealing with many tasks. Most existing meta-learning methods work based on the assumption that all tasks are of equal importance. However, real-world applications often present heterogeneous tasks characterized by varying difficulty levels, noise in training samples, or being distinctively different from most other tasks.
In this paper, we introduce a novel meta-learning method designed to effectively manage such heterogeneous tasks by employing rank-based task-level learning objectives, {\em \underline{He}terogeneous \underline{T}asks \underline{Ro}bust \underline{M}eta-learning} (HeTRoM). HeTRoM is proficient in handling heterogeneous tasks, and it prevents
easy tasks from overwhelming the meta-learner. The approach allows for an efficient iterative optimization algorithm based on bi-level optimization, which is then improved by integrating statistical guidance. Our experimental results demonstrate that our method provides flexibility, enabling users to adapt to diverse task settings and enhancing the meta-learner's overall performance.
\end{abstract}

%

\section{Introduction}
\label{sec:intro}

Deep neural network models trained with supervised learning have made significant progress in computer vision, in many cases surpassing human abilities. Conventional supervised learning needs a large number of training samples, but in real-world situations such as image classification \cite{russakovsky2015imagenet}, models must tackle multiple tasks that arise during the learning process, each with only a few samples available. This scenario is often referred to as {\em few-shot} learning. Meta-learning, or learning to learn, has been proposed as a solution to few-shot learning. In particular, optimization-based meta-learning methods introduce a meta-model \cite{finn2017model,huisman2021survey}, whose parameters can be fine-tuned for individual tasks using a few gradient descent update steps. In this way, the meta-model can use the knowledge gained from previous training tasks to enable the learning model to adapt rapidly to new tasks during meta-testing.

Current meta-learning methods often implicitly assume that all tasks are more or less equal in importance. However, in real-world applications, tasks are often heterogeneous in the level of learning difficulty, data quality, and domain consistency. We use image retrieval as an example to demonstrate the situation, Figure \ref{fig:noisy_setting}. Specifically, we can identify three types of tasks: (a) Harder tasks, which are tasks with ambiguous contexts (row 1: wolf vs. husky) are more challenging for the model; (b) Noisy tasks: tasks with corrupted labels (row 3, first sample mislabeled as ``boxer’’ and the query mislabeled as ``lion'') are more prone to errors; and (c) Outlier tasks that are considerably different from most training tasks (row 3 birds vs. row 1-2 airplanes). We can often differentiate these tasks using the task-level losses on actual multi-task meta-learning tasks, as shown in the bottom row of Figure \ref{fig:noisy_setting}. It shows a histogram that outlines the distribution of task-level losses by conducting fast adaptation using a MAML~\cite{finn2017model} trained meta-model on clean tasks from Mini-ImageNet, noisy (tasks derived from Mini-ImageNet but with labels randomly flipped), and outlier tasks with classes from Tiered-ImageNet that do not overlap with those in Mini-ImageNet tasks.

\begin{figure*}[t]
  \centering  
  \includegraphics[width=0.85\linewidth]{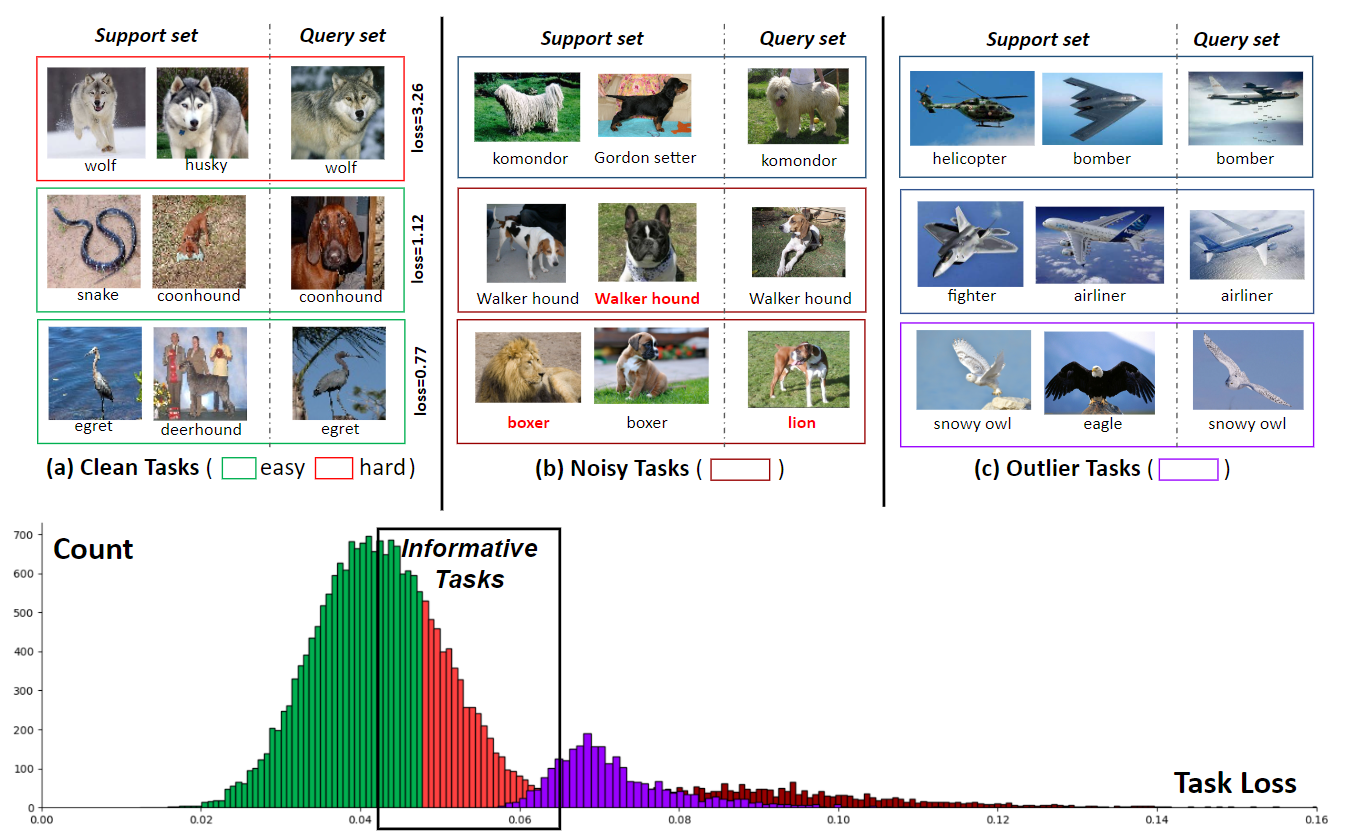}
  \caption{\it Heterogeneous tasks in meta-learning for image retrieval on real-world datasets (Mini-ImageNet and Tiered-ImageNet). {\bf (Top)} Three different types of tasks based on their difficulty. {\bf (Bottom)} Histograms of losses can be used to differentiate three types of tasks. 
  The displayed histogram is derived from real-world datasets, the details are provided in texts.
  }
  \label{fig:noisy_setting}
\end{figure*}

The heterogeneity introduces challenges for the meta-training step. 
Several works have concentrated on improving the effectiveness of meta-learning algorithms at the task level to address challenges posed by varying task difficulties \cite{collins2020task, nguyen2023task, liu2023not} and tasks with noise \cite{killamsetty2022nested, liu2020adaptive}.
However, limited works studied the effects of outlier tasks (also known as Out-of-Distribution tasks in \cite{killamsetty2022nested}) in meta-learning and the method of handling multiple heterogeneous tasks concurrently. A substantial body of literature has addressed the development of robust meta-learning algorithms to mitigate the impact of corrupted labels \cite{wang2020training, ren2018learning, killamsetty2022nested, lee2018cleannet, liu2020adaptive}. 
Nevertheless, these works predominantly concentrate on sample-level noise, with limited attention directed to examining robustness at task level.

In this paper, we introduce the {\em \underline{He}terogeneous \underline{T}asks \underline{Ro}bust \underline{M}eta-learning} (HeTRoM) method to address the problem posed by the heterogeneous tasks in meta-learning. The core component of HeTRoM is a new rank-based task-level loss that enables the dynamical selection of tasks in the meta-training step. Specifically, the HeTRoM loss controls the cut-off range in the ranked individual task losses, which reduces the influence of tasks whose losses are either too small (easy tasks) or too large (noisy tasks and outlier tasks). The hyper-parameters determining the range can be obtained based on a log-normal model of the distribution of the individual task losses. We show that the HeTRoM objective can be formulated as a bi-level optimization problem and efficiently solved with a stochastic gradient-based algorithm to update the meta-model. Experimental evaluations on two benchmark datasets underscore the effectiveness and robustness of HeTRoM.


\section{Related Works}

As the most representative optimization-based approach, Model Agnostic Meta-learning (MAML)~\cite{finn2017model} learns an initialization such that a gradient descent procedure starting from the initial model can achieve fast adaptation. In the following years, a large number of works on various MAML variants have been proposed ~\cite{grant2018recasting,finn2019online,
raghu2020rapid}. 

Another group of meta-regularization approaches has been proposed for improving the bias for a regularized empirical risk  
minimization problem~\cite{alquier2017regret, 
balcan2019provable,zhou2019efficient}. 
In addition, there is a common embedding-based framework in few-shot  learning~\cite{bertinetto2018meta,lee2019meta,
zhou2018deep}. The goal of this framework is to learn a shared embedding model that can be used for all tasks. Task-specific parameters are then learned for each task based on the embedded features. 

There exists a body of literature dedicated to investigating the robustness of optimization-based meta-learning. Some of these studies concentrate on addressing noise at the sample level within few-shot tasks.\cite{
killamsetty2022nested, 
mallem2023efficient, sun2022learning}.
For example, Eigen-Reptile \cite{chen2022robust} is a method that derives the updating direction of the meta-model by calculating the main direction of fast adaptation, which mitigates the impact of corrupted data within tasks. 

There are also several works focusing on task-level robustness \cite{yao2021meta, collins2020task, killamsetty2022nested, liang2022few}. \citealt{yao2021meta} propose to use an adaptive task scheduler (ATS) to decide the next tasks to be used while training to avoid the effect of tasks with corrupted labels. 
NestedMAML \cite{killamsetty2022nested} uses a weighted sum of instances or tasks in meta-training for handling Out-of-Distribution (OOD) tasks or noisy data, where the weights are updated through a nested bi-level optimization approach. While this re-weighting approach can reduce the effect of corrupted samples and tasks, it cannot eliminate them entirely. 
\citealt{liang2022few} focus on improving the robustness of metric-based meta-learning, which is not within the scope of this paper, as we concentrate on optimization-based meta-learning.

\section{Background} \label{sec:sensitive-motivation}

Meta-learning is commonly formulated as an optimization problem \cite{huisman2021survey}. Let $\{\mathcal{T}_i\}_{i=1}^n$ be a collection of $n$ tasks. Each task $\mathcal{T}_i$ has a specific loss function $\mathcal{L}(\phi, w_i; \xi)$ in sample $\xi$, where $\phi$ are the parameters of the shared embedding model for all tasks and $w_i$ are the task-specific parameters. We aim to learn the optimal parameters $\phi$ that work well for all tasks, and each task uses the optimized $\phi$ to adapt its specific parameters $w_i$ minimizing $\mathcal{L}(\phi, w_i; \xi)$. For instance, if the learning task is solved with a series of deep neural network models, $w_i$ may correspond to the parameters of the final layer within a neural network, while the parameters of the preceding layers are represented by $\phi$. For simplicity, we denote $\bw = (w_1,...,w_n)$ as the collection of task-specific parameters. The objective function of meta-learning is given by
\begin{equation}
    \begin{aligned}
    &\min_{\phi} L_{\mathcal{D}}(\phi, \bw^*)=\textstyle \frac{1}{n}\sum_{i=1}^n \ell_i(\phi) \\ 
    &\text{s.t.} \ \ \bw^* = \arg\min_{\bw}\mathcal{L}_{\mathcal{S}}(\phi, \bw):=\textstyle\frac{1}{n}\sum_{i=1}^n\mathcal{L}_{\mathcal{S}_i}(\phi, w_i),
    \end{aligned}
\label{eq:meta_learning}
\end{equation}
where $\ell_i(\phi)=\frac{1}{|\mathcal{D}_i|}\sum_{\xi\in\mathcal{D}_i}\mathcal{L}(\phi, w_i^*;\xi)$ and $\mathcal{L}_{\mathcal{S}_i}(\phi, w_i)=\frac{1}{|\mathcal{S}_i|}\sum_{\xi\in\mathcal{S}_i}\mathcal{L}(\phi, w_i;\xi)$. This is a bi-level optimization problem \cite{bracken1973mathematical, ji2021bilevel}, which can be solved in two stages. In the inner learning stage, the base learner of each task $\mathcal{T}_i$ searches the minimizer $w_i^*$ of its loss over a support set $\mathcal{S}_i$.  In the outer learning stage, the meta-learner evaluates the minimizers $w_i^*$ on a held-out query set $\mathcal{D}_i$ and optimizes $\phi$ across all tasks. 

As we mentioned in the Introduction, most existing meta-learning methods assume all tasks have equal influence on the meta-model. However, in real-world applications, tasks are often heterogeneous in terms of learning difficulty, data quality, and consistency. In the following experiments, we show the influence of heterogeneous tasks on meta-learning.

\noindent
\textbf{Influences of the Majority Tasks}.
\textit{\underline{Setting}:} We train two meta-models, one using all tasks in each mini-batch of size $16$ for updating meta-model parameters (traditional learning), and the other using only the top-8 tasks with the highest individual task losses (hard task mining). We evaluate both models by fast adaptation to all training tasks and testing them on the same test dataset. The distributions of task query losses and the test accuracy are presented in Figure \ref{fig:sensitivity} (a).
\textit{\underline{Analysis}:} Using hard task mining during meta-training leads to lower task losses compared to using all tasks, indicating that easy tasks dominate training. Moreover, training with ``hard'' tasks improves the generalization of the meta-model, as demonstrated by the increased testing accuracy.  

\begin{figure*}[t]
  \centering
  \includegraphics[width=0.95\linewidth]{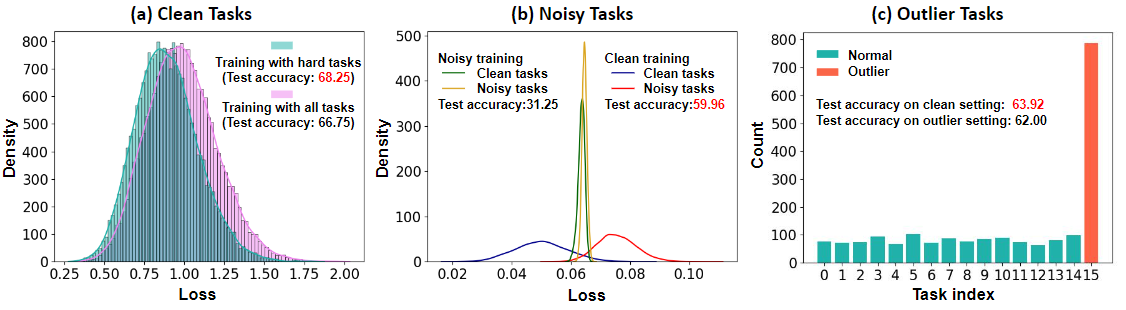}
  \caption{\it Behavior analysis of meta-learning. (a) The loss distributions of training tasks with models trained with all tasks and hard tasks, respectively. (b) The loss distributions of clean and noisy tasks with models trained with a clean dataset (``Clean training``) and a noisy dataset (``Noisy training``), respectively.  (c) The frequency of each task index with the highest loss throughout the training process, where the outlier task is assigned index 15 in each iteration. "Test accuracy" (\%) is obtained by testing on a separate clean test set.}
  \label{fig:sensitivity}
\end{figure*}

\noindent
\textbf{Sensitive to Noisy Tasks}. 
\textit{\underline{Setting}:} We generate two training sets: clean and noisy, where 60\% of the tasks contain samples with flipped labels.
Then, we trained two meta-models on these two datasets separately and conducted fast adaptation on all clean tasks and noisy tasks respectively in the noisy training set. The distributions of training task losses together with test accuracy are shown in Figure \ref{fig:sensitivity} (b).
\textit{\underline{Analysis}:} 
When trained with a noisy dataset, there is a significant overlap between the task loss distributions of noisy and clean tasks (yellow and green lines, respectively). However, when trained with a clean dataset, there is a clear separation between the two distributions, indicating that the meta-model trained with the noisy dataset cannot effectively distinguish between noisy and clean tasks.
Additionally, training with noisy datasets results in higher clean task losses than clean training (blue line), while clean training results in higher noisy task losses than noisy training (red line). These illustrate the susceptibility of meta-learning to noisy tasks and how it can harm the meta-model's generalization ability, as indicated by the decline in test accuracy.

\noindent
\textbf{Vulnerability to the Outlier Tasks}. 
\textit{\underline{Setting}:} Two categories with multiple classes from the Tiered-ImageNet \cite{ren2018meta} dataset were randomly chosen as the target and outlier categories. Two meta-models were trained, the first one using 16 tasks randomly sampled from the target category as clean tasks in each iteration of the mini-batch, and the second one using 15 tasks from the target category (indexed 0 to 14) and 1 task from the outlier category (indexed as 15) in each iteration of the mini-batch. For the second one, we track the task with the highest loss in each iteration by recording its task index, and the frequency of each task index being the highest throughout the training process is plotted. The test accuracy of both models is also reported in Figure \ref{fig:sensitivity} (c). \textit{\underline{Analysis}:} The task from the outlier category frequently has the highest loss in each batch. 
Meta-learning is also vulnerable to outlier tasks, as evidenced by the lower test accuracy compared to the model trained without outliers.

\section{Method}
\label{sec:method}

HeTRoM is based on a simple observation that hard, noisy, or outlier tasks are usually associated with large task-level losses during meta-training, while most tasks have relatively small losses throughout the training process.
A similar observation can be made at the sample level in \cite{zhu2023robust, jiang2018mentornet, wang2020training}. As illustrated in Section \ref{sec:sensitive-motivation}, our findings also show that this phenomenon exists widely at the task level.
Let $\mathbb{L}:=\{\ell_1(\phi),...,\ell_n(\phi)\}$ be a set of all task losses on their corresponding query dataset. We also denote $\ell_{[i]}(\phi)$ as the $i$-th largest individual task loss after sorting the elements in set $\mathbb{L}$ (ties can be broken in any consistent way). With two integers $k$ and $m$, $0\leq m<k\leq n$, the HeTRoM loss is given by: 
\begin{equation}
    \begin{aligned}
    &\min_{\phi} \mathcal{L}_{\mathcal{D}}(\phi, \bw^*)=\textstyle\frac{1}{k-m}\sum_{i=m+1}^k \ell_{[i]}(\phi) \\
    &\text{s.t.} \ \ \bw^* = \arg\min_{\bw}\mathcal{L}_{\mathcal{S}}(\phi, \bw):=\textstyle\frac{1}{n}\sum_{i=1}^n\mathcal{L}_{\mathcal{S}_i}(\phi, w_i).
    \end{aligned}
\label{eq:aorr_meta}
\end{equation}
The HeTRoM loss is a generalization of ranked ranges in supervised learning \cite{hu2020learning}. It is equivalent to a dynamic task selection method that filters imperfect tasks based on their individual task losses. Specifically, it selects tasks with small losses, which are more likely to be normal tasks, and drops the top-$m$ tasks with large. It also excludes tasks with small losses (\ie, top-$k$ to top-$n$ losses), which are the majority of tasks that are normal and potentially redundant. Conventional meta-learning is a special case of HeTRoM with $k=n$ and $m=0$.

\subsection{Algorithm}
The ranking operation in HeTRoM can be eliminated using the following result. 
\begin{theorem}
Denote $[a]_+=\max\{0,a\}$ as the hinge function. Equation (\ref{eq:aorr_meta}) is equivalent to
\begin{align}\label{eq:aorr_meta_minimax}
    \min_{\phi, \lambda} \max_{\hat{\lambda}}& \hat{\mathcal{L}}_{\mathcal{D}}(\phi,\bw^*)= \textstyle\frac{1}{k-m}\sum_{i=1}^n\hat{\mathcal{L}}(\phi, w_i^*,\lambda,\hat{\lambda})\nonumber\\
    &:=\Big[[\ell_i(\phi)-\lambda]_+-[\ell_i(\phi)-\hat{\lambda}]_++\frac{k}{n}\lambda-\frac{m}{n}\hat{\lambda}\Big]\nonumber\\
    \emph{s.t.} \ \ \bw^* &= \arg\min_{\bw}\mathcal{L}_{\mathcal{S}}(\phi, \bw):=\textstyle\frac{1}{n}\sum_{i=1}^n\mathcal{L}_{\mathcal{S}_i}(\phi, w_i),
\end{align}
If the optimal of $\phi$ and $\bw^*$ are achieved, $(\ell_{[k]}(\phi),\ell_{[m]}(\phi))$ can be the optimum solutions of $(\lambda,\hat{\lambda})$.
\label{theorem1}
\end{theorem}

The proof of Theorem \ref{theorem1} can be found in Appendix \ref{sec:proofs}. A bi-level optimization procedure \cite{ji2020bilevel} can be used to solve (\ref{eq:aorr_meta_minimax}). Specifically, in each iteration $j$, we sample a batch of $N$ tasks $\{\mathcal{T}_i\}_{i=1}^N$, and then each sampled task $\mathcal{T}_i$ conducts $D$ steps of inner-loop updates as 
\begin{align}
    w_{i,j}^t = w_{i,j}^{t-1}-\alpha \nabla_{w} \mathcal{L}_{\mathcal{S}_i}(\phi_j,w_{i,j}^{t-1}),\quad t=1,....,D 
\end{align}
where $\alpha$ is the inner-loop step size, $w_{i,j}^t$ is the task $\mathcal{T}_i$'s parameter at iteration $j$ in step $t$, and $\nabla_{w} \mathcal{L}_{\mathcal{S}_i}$ is the gradient of $\mathcal{L}_{\mathcal{S}_i}$ with respect to (w.r.t.) $w$. Then, we need to update $\lambda$ and $\hat{\lambda}$ together with the meta-model to keep track of the ranked range of task losses on the query set from top-$m$ to top-$k$. To optimize the minimax problem in upper-level objective of equation~\eqref{eq:aorr_meta_minimax}, we use single gradient steps to update $\lambda$ and $\hat{\lambda}$ simultaneously using their (sub)gradients given by 
\begin{equation}
    \begin{aligned}
        &\partial_{\lambda}\hat{\mathcal{L}}(\phi_j, w_{i,j}^D,\lambda_j,\hat{\lambda}_{j})= \frac{k}{N} - \mathbb{I}_{[\ell_i(\phi_j, w_{i,j}^D)>\lambda_{j}]}
        , \\ &\partial_{\hat{\lambda}}\hat{\mathcal{L}}(\phi_j, w_{i,j}^D,\lambda_j,\hat{\lambda}_j) =  \mathbb{I}_{[\ell_i(\phi_j, w_{i,j}^D)>\hat{\lambda}_j]}-\frac{m}{N}
        \label{eq:grad_lambda_lambda_hat}
    \end{aligned}
\end{equation}
where $\ell_i(\phi_j, w_{i,j}^D)$ is the task $\mathcal{T}_i$'s loss on its query set using iteration $j$'s meta-model parameters $\phi_j$ and task parameter $w_{i,j}^D$, and $\mathbb{I}_{[A]}$ is an indicator function, which returns $1$ if the condition $A$ is true. 
Finally, we adopt the stable iterative differentiation (ITD)-based approach~\cite{grazzi2020iteration} to efficiently approximate the gradient of $\hat{\mathcal{L}}_{\mathcal{D}}(\phi,\bw^*)$ w.r.t.~$\phi$ (which refers to as the hyper-gradient because of the implicit dependence of $\mathbf{w}^*$ on $\phi$). In specific, 
we first evaluate the upper-level function 
over the inner-loop output $w_{i,j}^{D}$, and compute the hyper-gradient estimate as  
\begin{align}\label{eq:gradiant_meta_model}
\text{Hg}(\phi_j):=\frac{1}{k-m}\frac{\partial \sum_{i=1}^N\hat{\mathcal{L}}(\phi_j, w^D_{i,j},\lambda_j,\hat{\lambda}_j)}{\partial \phi_j}
\end{align}
by conducting back-propagation over inner-loop trajectory w.r.t.~$\phi_j$. Following \cite{finn2017model,ji2021bilevel}, the first-order hyper-gradient approximation $\frac{1}{k-m}\sum_{i=1}^N\partial_\phi\hat{\mathcal{L}}(\phi_j, w^D_{i,j},\lambda_j,\hat{\lambda}_j)$ is employed to avoid the computation cost introduced by computing implicit derivative $\frac{\partial w_{i,j}^D}{\partial \phi_i}$ that contains second-order model information with a large model size.  
We show the entire optimization procedure of our method in Algorithm \ref{alg:1}.

\begin{algorithm}[t]
	\renewcommand{\algorithmicrequire}{\textbf{Require:}}
	\caption{Optimization procedure of HeTRoM}
	\label{alg:1}
	\begin{algorithmic}[1]
 \STATE {\bfseries Input:}  Total iteration number $J$; inner-loop iteration number $D$; batch size $N$; stepsizes $\alpha, \beta, \gamma $, initial meta-parameter $\phi_0$; initial task-specific parameter $w_0$; hyper-parameters $k, m$; $\lambda_0, \hat{\lambda}_0$.
            \FOR{$j=0,1,2,...,J$}
            \STATE{Sample a batch of tasks $\{\mathcal{T}_i\}_{i=1}^N$}
            \FOR{each task $\mathcal{T}_i$ in $\{\mathcal{T}_i\}_{i=1}^N$}
                \STATE{Initialize task-specific parameter as $w_{i,j}^{0} = w_0$}
                \FOR{$t=1,2,...,D$}
                \STATE{Update $w_{i, j}^t = w_{i, j}^{t-1}-\alpha \nabla_{w} \mathcal{L}_{\mathcal{S}_i}(\phi_j,w_{i, j}^{t-1}) $}
                \ENDFOR
            \ENDFOR
                \STATE{
                \scalebox{0.8}{
                $\begin{pmatrix}
                    \lambda_{j+1}\\ 
                    \hat{\lambda}_{j+1}
                    
                    \end{pmatrix} \leftarrow \begin{pmatrix}
                    \lambda_{j}\\ 
                    \hat{\lambda}_{j}
                    \end{pmatrix}-\frac{\gamma}{k-m} \begin{pmatrix}
                    \sum_{i=1}^N \partial_{\lambda}\hat{\mathcal{L}}(\phi_j, w_{i,j}^D,\lambda_j,\hat{\lambda}_{j})\\ 
                    -\sum_{i=1}^N \partial_{\hat{\lambda}}\hat{\mathcal{L}}(\phi_j, w_{i,j}^D,\lambda_j,\hat{\lambda}_{j})
                    \end{pmatrix}$
                } based on equation (\ref{eq:grad_lambda_lambda_hat})}
		        \STATE{Update meta-parameter  $\phi_{j+1}=\phi_j - \beta\text{Hg}(\phi_j)$ based on equation~\eqref{eq:gradiant_meta_model}}
            \ENDFOR
        \end{algorithmic}  
\end{algorithm}

In practice, we follow the results obtained from Theorem \ref{theorem1} and initialize $\lambda$ and $\hat{\lambda}$ through adaptation on a batch of randomly sampled tasks. Then, we sort the loss on the corresponding query sets and assign the $k$-th loss and $m$-th loss values to $\lambda$ and $\hat{\lambda}$, respectively.

\subsection{Determining the hyper-parameters}
\label{sec:statistic-guided-learning}

The values of hyper-parameters $k$ and $m$ are important to the overall performance of the model trained by Algorithm \ref{alg:1}. We describe three practical measures to find their values. 

\noindent
\textbf{Warm-up}. Based on Theorem \ref{theorem1}, we can determine how many task losses exceed the values of $\lambda$ and $\hat{\lambda}$ by analyzing the values of $k$ and $m$. Specifically, we use the conventional meta-learning approach to ``warm up'' the model based on equation (\ref{eq:meta_learning}) for a few iterations by training on the training data, which may contain noisy or outlier tasks. This procedure is similar to Algorithm \ref{alg:1}. In particular, we have no $\lambda$ and $\hat{\lambda}$ that need to be updated but we update meta-parameter $\phi$ using $\phi_{j+1}=\phi_j - \beta\overline{\text{Hg}}(\phi_j)$, where $\overline{\text{Hg}}(\phi_j):=\frac{1}{N}\frac{\partial \sum_{i=1}^N \ell_i(\phi_j, w^D_{i,j})}{\partial  \phi_j}$. 

\noindent
\textbf{Identification}. After the ``warm-up'', we fast adapt the meta-model with parameters $\phi_{j+1}$ to all training tasks $\{\mathcal{T}_i\}_{i=1}^n$ using $D$ steps of the inner-loop for each of them. Then we obtain the new trained task-specific parameters $\{\Tilde{w}_i^D\}_{i=1}^n$. Next, we calculate all training task losses on their corresponding query sets so that we get $\{\ell_i(\phi_{j+1}, \Tilde{w}_i^D)\}_{i=1}^n$. A histogram based on $\{\ell_i(\phi_{j+1}, \Tilde{w}_i^D)\}_{i=1}^n$ is used to describe the training behaviors of all tasks. We can use different distributions to fit the histogram. However, in practice, the log-normal distribution can be a better fit. In addition, it is also popular to use log-normal distribution to fit training loss as \cite{holland2023flexible} in machine learning. Therefore, we take a log of all task losses as $\overline{\mathbb{L}}:=\{\log\ell_i(\phi_{j+1}, \Tilde{w}_i^D)\}_{i=1}^n$ and use their mean $\text{Mean}(\overline{\mathbb{L}})$ and standard deviation $\text{Std}(\overline{\mathbb{L}})$ as the descriptor for the statistic of task losses. In order to identify potential easy tasks and noisy or outlier tasks, we must define the thresholds that differentiate them from others. Specifically, we define the thresholds as:
\begin{align}\label{eq:two-rhos}
    &\lambda^*= \exp(\text{Mean}(\overline{\mathbb{L}}) - \rho_1\text{Std}(\overline{\mathbb{L}})), \\
    &\hat{\lambda}^* = \exp(\text{Mean}(\overline{\mathbb{L}}) + \rho_2\text{Std}(\overline{\mathbb{L}})),
\end{align}
where $\rho_1$ and $\rho_2$ are hyper-parameters. 
In practice, we find that setting them within the range $[1, 3]$ yields good results.

\noindent
\textbf{Distillation}. After we get the thresholds, we continue to train the meta-model obtained from the warm-up procedure. Thus, the meta-model parameters should be $\phi_{j+1}$. Then, we use the thresholds to select tasks to update the meta-model parameters dynamically. In particular, tasks with their losses fall into [$\lambda^*$, $\hat{\lambda}^*$] are regarded as useful tasks that benefit learning.  The learning procedure is also similar to Algorithm \ref{alg:1}. However, in the upper-level loop, we only need to use the following gradient to update meta-model parameters $\phi$:
\begin{align}
    \widetilde{\text{Hg}}(\phi_{j+1}):= \textstyle\frac{1}{|\mathcal{S}_{\ell}|} \frac{ \partial \sum_{\ell_i \in \mathcal{S}_{\ell}} \ell_i}{\partial \phi_{j+1}}
    \label{static-gradient}
\end{align}
where $\mathcal{S}_{\ell}=\mathbb{I}_{[\lambda^*\leq \ell_i(\phi_{j+1}, w_{i,j}^D) \leq \hat{\lambda}^*]}$ is the set containing task losses that fall between $\lambda^*$ and $\hat\lambda^*$.

The corresponding Algorithm \ref{alg:2} is available in Appendix \ref{sec:alg2}. We call this approach statistic-guided HeTRoM (sg-HeTRoM). The significance of the thresholds can be more clearly understood as follows: $\lambda^*$ and $\hat{\lambda}^*$ can be regarded as the optimal values of $\lambda$ and $\hat{\lambda}$. In each iteration of the distillation process, we eliminate the impact of noisy or outlier tasks on the meta-parameter updates by using the upper bound $\hat{\lambda}^*$ of the task losses. Additionally, we enhance the meta-model's generalization ability by potentially removing overly simple tasks based on the lower bound $\lambda^*$, which is akin to a hard task mining strategy.

\begin{table*}[t]
\renewcommand\arraystretch{1.1}
\center
\scalebox{0.68}{
\begin{tabular}{c|c|c|cc|ccc|ccc}
\hline
 \multirow{2}{*}{Dataset}& \multirow{2}{*}{Model}& \multirow{2}{*}{Method} & \multicolumn{2}{c|}{Clean} &  \multicolumn{3}{c|}{Noisy (5-ways 5-shots)} & \multicolumn{3}{c}{Outlier (3-ways 5-shots)} \\
&   &  & 5-ways 5-shots & 5-ways 1-shot  & Fixed  &  1-step  & 2-step   &  ratio=0 & ratio=0.1  & ratio=0.3\\ \hline
\multirow{8}{*}{\rotatebox{90}{\makecell{Mini-\\ImageNet}}} & \multirow{5}{*}{CNN4} & MAML \cite{finn2017model} &$63.11 \pm 0.91 $ & $48.70\pm 1.75 $ & $60.25$ & $55.20$ & $30.30$ & $78.05$ & $77.44$ & $77.29$ \\
& &  ATS \cite{yao2021meta} & $61.60 \pm 0.69$ & $50.75 \pm 0.73$ & $59.89$ & $60.09$ & $59.50$ & $71.09$ & $70.80$ & $69.84$ \\
& & Eigen-Reptile\cite{chen2022robust} & $69.85 \pm 0.85   $ & $53.25 \pm 0.45 $ & $63.47$ & $61.78$ & $49.63$ & $69.56$ & $69.50$ & $67.95$ \\ \cline{3-11}

& & HeTRoM (Ours) & \colorbox{llgray}{$\textbf{70.45} \pm \textbf{0.71 }$} &  \colorbox{llgray}{$59.75  \pm 1.28$} & \colorbox{llgray}{$\textbf{68.75}$} & \colorbox{llgray}{$67.30$} & \colorbox{llgray}{$61.45$} & \colorbox{llgray}{$\textbf{81.25}$} & \colorbox{llgray}{$78.96$} & \colorbox{llgray}{$\textbf{80.63}$}\\
& & sg-HeTRoM (Ours) & \colorbox{llgray}{$70.08 \pm 1.07$} & \colorbox{llgray}{$\textbf{60.50} \pm \textbf{2.03}$} & \colorbox{llgray}{$67.58$} & \colorbox{llgray}{$\textbf{68.00}$} & \colorbox{llgray}{$\textbf{67.67}$} & \colorbox{llgray}{$79.86$} & \colorbox{llgray}{$\textbf{79.17}$} & \colorbox{llgray}{${78.12}$} \\ \cline{2-11}

& \multirow{3}{*}{ResNet12} & MAML \cite{finn2017model} & $69.54 \pm 0.38$ & $58.60 \pm 0.42$ & $64.50$ & $58.25$ & $44.75$ & $79.30$ & $77.78$ & $77.29$ \\ \cline{3-11}
& &  HeTRoM (Ours) & \colorbox{llgray}{$\textbf{75.35} \pm \textbf{0.82}$} & \colorbox{llgray}{$\textbf{60.50} \pm \textbf{2.32}$} & \colorbox{llgray}{$\textbf{71.25}$} & \colorbox{llgray}{$\textbf{64.50}$} & \colorbox{llgray}{$\textbf{54.33}$} & \colorbox{llgray}{$\textbf{87.35}$}& \colorbox{llgray}{$\textbf{86.21}$} & \colorbox{llgray}{$\textbf{80.77}$}\\ 
& & sg-HeTRoM (Ours) & \colorbox{llgray}{$71.87 \pm 1.93$ } & \colorbox{llgray}{$59.38 \pm 1.84$} & \colorbox{llgray}{$70.13$} & \colorbox{llgray}{$\textbf{67.25}$} & \colorbox{llgray}{$\textbf{59.75}$} & \colorbox{llgray}{$86.18$} & \colorbox{llgray}{$85.50$} & \colorbox{llgray}{$80.01$} \\
\hline

\multirow{8}{*}{\rotatebox{90}{\makecell{Tiered-\\ImageNet}}} & \multirow{5}{*}{CNN4} & MAML \cite{finn2017model} &  $66.25 \pm 0.19 $ & $50.98 \pm 0.26 $ & $34.50$ & $31.25$ & $31.50$ & $63.92$ & $62.83$ & $62.00$  \\
& & ATS\cite{yao2021meta} &  $60.44 \pm 0.76$ & $43.25 \pm 0.84$ & $58.72$ & $57.12$ & $49.06$ & $48.69$ & $48.64$ & $48.02$ \\
& & Eigen-Reptile \cite{chen2022robust} & $57.36 \pm 0.19$ & $42.24 \pm 0.18$ & $56.05$ & $53.37$ & $52.16$ & $59.01$ & $57.17$ & $55.41$ \\ \cline{3-11}
& & HeTRoM (Ours) & \colorbox{llgray}{$\textbf{71.25} \pm \textbf{0.57 }$} & \colorbox{llgray}{$58.50 \pm 1.07 $} & \colorbox{llgray}{$\textbf{68.75}$} & \colorbox{llgray}{$66.00$} & \colorbox{llgray}{$\textbf{67.75}$} & \colorbox{llgray}{$67.07$} & \colorbox{llgray}{$66.75$} & \colorbox{llgray}{$\textbf{65.50}$} \\
& & sg-HeTRoM (Ours) & \colorbox{llgray}{$70.80 \pm 1.74$} & \colorbox{llgray}{$\textbf{58.75} \pm \textbf{0.98}$} & \colorbox{llgray}{$68.67$} & \colorbox{llgray}{$\textbf{69.33}$} & \colorbox{llgray}{$66.83$}& \colorbox{llgray}{$\textbf{69.17}$} & \colorbox{llgray}{$\textbf{69.37}$} & \colorbox{llgray}{$64.17$}\\ \cline{2-11}
& \multirow{3}{*}{ResNet12} & MAML \cite{finn2017model} & $71.24 \pm 0.43$ & $58.58 \pm 0.49$ & $67.00$ & $68.25$ & $54.63$ & ${65.83}$ & $61.33$ & $ 51.67$  \\ \cline{3-11}
& & HeTRoM (Ours) & \colorbox{llgray}{$\textbf{72.25} \pm \textbf{1.29}$} & \colorbox{llgray}{$\textbf{61.25} \pm \textbf{0.82}$} & \colorbox{llgray}{$69.17$} & $64.88$ & \colorbox{llgray}{$67.75$} & \colorbox{llgray}{$66.25$} & \colorbox{llgray}{$\textbf{65.00}$} & \colorbox{llgray}{$58.75$} \\
& & sg-HeTRoM (Ours) & \colorbox{llgray}{$71.25 \pm 0.74$} & \colorbox{llgray}{$58.75 \pm 1.64$} & \colorbox{llgray}{$\textbf{69.33}$} & \colorbox{llgray}{$\textbf{70.25}$} & \colorbox{llgray}{$\textbf{68.25}$} & \colorbox{llgray}{$\textbf{67.08}$} & \colorbox{llgray}{$64.17$}  & \colorbox{llgray}{$\textbf{59.16}$}\\
\hline
\end{tabular}
}
\caption{\label{table1} \textit{Test accuracy (\%) on Mini-ImageNet and Tiered-ImageNet. For the clean setting, we report the average accuracy with 95\% confidence intervals. 
The best results are shown in \textbf{Bold}. \colorbox{llgray}{Gray} indicates our methods outperform baseline methods in the same setting.}}  
\end{table*}

\section{Experiments}



\subsection{Experimental Settings}
\label{experimental-settings}

\noindent
\textbf{Datasets.} {We conduct our experiments in a few-shot classification setting, where two widely recognized datasets are used for few-shot classification: Mini-ImageNet \cite{vinyals2016matching}, and Tiered-ImageNet \cite{ren2018meta}.} Both are subsets of the ILSVRC-12 dataset \cite{russakovsky2015imagenet}. Mini-ImageNet comprises $100$ classes, with each class containing $600$ images of $84 \times 84$ pixel dimensions. The dataset is divided into three subsets: $64$ classes for training, $16$ for validation, and $20$ for testing. On the other hand, Tiered-ImageNet is a more substantial dataset, containing $779,165$ images annotated across $608$ classes. These classes are further grouped into $34$ categories. This dataset is split at the category level, with $20$ categories used for training, $6$ for validation, and $8$ for testing.

\noindent
\textbf{Baselines.} We showcase the effectiveness of our approach by contrasting it with recent optimization-based meta-learning methods \textbf{aiming at enhancing the robustness to heterogeneous tasks} of meta-training, such as ATS \cite{yao2021meta} and Eigen-Reptile \cite{chen2022robust}, Consequently, we do not directly compare our method with strategies that solely aim at boosting the meta-model's performance~\cite{baik2021meta, baik2020meta, collins2020task}. NestedMAML \cite{killamsetty2022nested} is not included in our comparison, as the original paper neither provides source code nor records performance on prevalent datasets like Mini-ImageNet and Tiered-ImageNet. Moreover, we examine the impact of our HeTRoM loss on the issues previously mentioned in the context of representation-based meta-learning methods. Importantly, our investigation is conducted without relying on training tracks, as is the case in \cite{antoniou2018train}. As a result, we benchmark our method against the MAML approach, which is widely acknowledged as a prevalent method in meta-learning. We do not compare with the most recent optimization-based meta-learning methods because our approach is versatile and can readily incorporate these methods. 
The detailed information on implementation can be found in Appendix \ref{sec:add-exp-details}.


\subsection{Results}

\noindent
\textbf{Performance on Clean Tasks.}
We report the overall results (accuracy on the testing set) based on clean tasks in Table \ref{table1}. Specifically, we use the HeTRoM loss and the statistically driven method to filter out non-informative and easy tasks during training dynamically. We achieve a significant improvement in performance over the MAML method, with a $[1.9\%,11.8\%]$ margin. In addition, our method outperforms two robust meta-learning methods ATS and Eigen-Reptile on both datasets and two different few-shot settings (\ie, 5-ways 5-shots and 5-ways 1-shot).

\begin{figure}[t]
  \centering
  \includegraphics[width=1\linewidth]{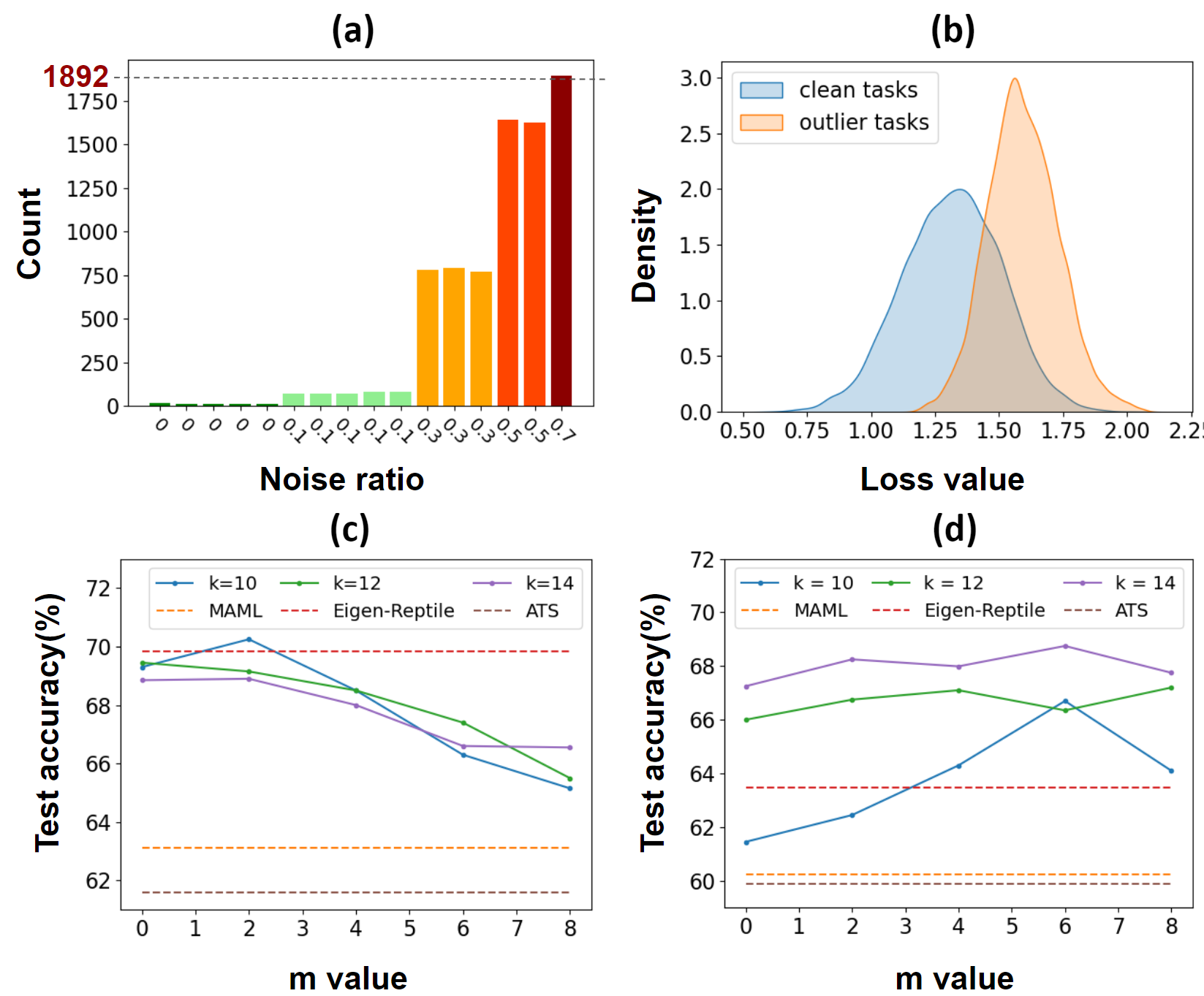}
  \caption{\it (a) Frequency of the noise ratio for each task in a mini-batch being excluded by $\hat{\lambda}$ during meta-training. (b) Loss distribution of clean and outlier tasks when conducting fast adaptation with a trained meta-model. (c) Test accuracy with varying $k$ and $m$ in clean task setting. (d) Test accuracy with varying $k$ and $m$ in fixed noisy task setting.}
  \label{fig:result_figure}
\end{figure}

\noindent
\textbf{Performance on Noisy Tasks.} Table \ref{table1} also shows the results of all methods on noisy task settings of 5-ways 5-shots classification. It is clear that all methods experience a significant decrease in performance in the presence of noise compared to the clean task scenario, but our methods still perform better than the others. As we move from Mini-ImageNet to Tiered-ImageNet, the variety of tasks increases and MAML loses its generalization ability. This may be attributed to the fact that the number of easy tasks has increased significantly.
On the other hand, ATS and Eigen-Reptile perform relatively well but still show a decreasing trend. In contrast, our methods consistently maintain stable or even better performance, indicating their robustness to the diversity of tasks in the presence of noisy tasks. 

We then analyze in detail how HeTRoM works to enhance the robustness of MAML in a noisy task scenario under the \textit{fixed} setting. We can keep track of the noise ratio for each task within a mini-batch since each task is indexed and fixed with one noisy ratio. This means that we can record the frequency of the noise ratio for each task excluded by $\hat\lambda$ (\ie, identified as a noisy task) during meta-training. Figure \ref{fig:result_figure} (a) displays this information. Our HeTRoM method can identify noisy tasks during meta-training, as demonstrated by the higher frequency of exclusion of tasks with a higher noise ratio. Clean tasks are rarely excluded, providing valuable information for training. Note that we do not require our method to exclude noisy tasks precisely in each iteration, which is not practical for an efficient method.

\noindent
\textbf{Performance on Outlier Tasks.} Table \ref{table1} presents a comparison of our methods with other approaches on the outlier tasks setting of 3-ways 5-shots. We can observe that the performance of MAML shows a decreasing trend when increasing the ratio of outlier tasks. ATS and Eigen-Reptile fail to achieve competitive performance with MAML on both datasets, while our method demonstrates superior performance compared to MAML by a substantial margin. This can be attributed to the effective filtering of outlier tasks and too-easy tasks.
To provide additional validation for the efficacy of removing outlier tasks during meta-training using our approach, we randomly sample 10,000 tasks from the Mini-ImageNet dataset (considered as clean) and the Tiered-ImageNet dataset (considered as outlier) respectively. A meta-model trained by our method performs fast adaptation on all these tasks, and the distribution of task losses is depicted in Figure \ref{fig:result_figure} (b). The results exhibit a clear trend wherein the task loss of outlier tasks tends to be larger than that of clean tasks, justifying the use of loss as a criterion for identifying outlier tasks.

\begin{table}[t]
    \renewcommand\arraystretch{1.2}
\centering
    \scalebox{0.6}{
    \begin{tabular}{c|c|c|ccc}
        \hline
        \multirow{2}{*}{Architecture} & \multirow{2}{*}{Setting} & \multirow{2}{*}{Clean} & \multicolumn{3}{c}{Noisy }                          \\ 
                                      &                          &                        & Fixed & 1-step & 2-step \\ \hline
        \multirow{2}{*}{CNN4}         & w/o HeTRoM              & $69.61\pm 0.38$                & $67.36 \pm 0.38$          & $67.06 \pm 0.39$           & $63.35 \pm 0.36$           \\
                                      & w/HeTRoM      & $\textbf{69.91} \pm \textbf{0.37}$                  & $\textbf{67.65} \pm \textbf{0.39}$          & $\textbf{69.23}\pm \textbf{0.38}$           & $\textbf{65.23} \pm \textbf{0.38}$           \\ \hline
        \multirow{2}{*}{ResNet12}     & w/o HeTRoM               & $78.63 \pm 0.46$                  & $76.76 \pm 0.40$         & $75.18 \pm 0.40$          & $72.40 \pm 0.38$          \\
                                      & w/HeTRoM      & $\textbf{78.97}\pm\textbf{0.34}$                 & $\textbf{78.61}\pm\textbf{0.35}$          & $\textbf{78.45}\pm\textbf{0.35}$           & $\textbf{75.95}\pm\textbf{0.36}$          \\ \hline
    \end{tabular}
    }
\caption{\label{tab:res-metaoptnet} \textit{Test accuracy (\%) on Mini-ImageNet when incorporated with MetaOptNet \cite{lee2019meta}. \textbf{Bold} indicates best results for each setting.}}  
\end{table}

\noindent
\textbf{Performance of sg-HeTRoM method.} 
From Table \ref{table1}, we can also observe the performance comparison between models trained with HeTRoM and sg-HeTRoM. The results indicate no significant performance gap between two methods. Moreover, both methods outperform all baselines, highlighting the effectiveness of leveraging statistical information to identify informative tasks and filter out noisy and outlier tasks. Furthermore, using a large model ResNet12 leads to improved performance in most settings. Nevertheless, our method still outperforms MAML.

\noindent
\textbf{Incorporating with SOTA Method.} To ascertain the efficacy of our approach in conjunction with other state-of-the-art (SOTA) methodologies, we integrated our method with MetaOptNet\cite{lee2019meta} in both clean task and noisy task settings using Mini-ImageNet dataset. The results are shown in Table \ref{tab:res-metaoptnet}. It is evident from the results that the performance of MetaOptNet experiences a significant decline when confronted with noisy tasks. In contrast, our method demonstrates a mitigating effect on this performance degradation, particularly noteworthy in experiments with larger model (ResNet 12 in Table \ref{tab:res-metaoptnet}). More results on incorporating other SOTA methods can be found in Appendix \ref{sec:add-exp-results}.

\subsection{Discussion}
\noindent
\textbf{$k$ and $m$ in Clean Setting.} We analyze the effect on training with a variety of $k$ and $m$ in the clean task setting. As common settings, we conduct meta-training on the original Mini-ImageNet dataset with $k$ chosen from $\{10, 12, 14\}$, and $m$ chosen from $\{0, 2, 4, 6, 8\}$. The results are shown in Figure \ref{fig:result_figure}(c). The experimental results suggest two key findings:  firstly, the best performance for each $k$ is achieved with a small $m$ value (\eg, $m=2$), and performance diminishes with increasing $k$, indicating that easy tasks mainly influence meta-learning. Secondly, increasing $m$ for a fixed $k$ results in performance declines, underscoring the beneficial effect of incorporating hard tasks in training.


\noindent
\textbf{$k$ and $m$ in Noisy Setting.} We study the impact of hyper-parameters in the noisy task settings by conducting experiments in \textit{fixed} noisy setting with a variety of $k$ and $m$ on Mini-ImageNet. The results are shown in Figure \ref{fig:result_figure} (d). As $m$ increases, there is a clear trend of improved performance, which can be attributed to the exclusion of more corrupted tasks from training, given the fact that the majority of tasks in a mini-batch are affected by varying levels of noise in \textit{fixed} noise setting. Additionally, a larger value of $k$ generally leads to higher performance, as it includes more clean tasks during training. Moreover, there is a clear range of $k$ and $m$ that our method outperforms the compared methods.

\noindent
\textbf{$\rho_1$ and $\rho_2$ in Outlier Setting.} We conducted experiments on the Tiered-ImageNet dataset with the outlier setting (ratio=0.1) using sg-HeTRoM to study the impact of hyper-parameters $\rho_1$ and $\rho_2$ in equation (\ref{eq:two-rhos}), which determine the thresholds for identifying easy tasks and outlier tasks. The results are presented in Table \ref{table:rho}. 
We observe that sg-HeTRoM is not overly sensitive to $\rho_1$ and $\rho_2$. 
If we set $\rho_2=\infty$, changing the value of $\rho_1$ does not seem to affect the performance significantly. This could be because the outlier tasks significantly impact meta-learning more than the easy tasks. However, when we set $\rho_1=\infty$, and $\rho_2$ varies, there is a visible fluctuation in the performance. This implies that $\rho_2$ is useful in mitigating the effect of outlier tasks.

\begin{table}[t]
    \renewcommand\arraystretch{1.2}
\centering
    \scalebox{0.8}{
        \begin{tabular}{c|cccc} 
            \hline
            & $\rho_1=1$ & $\rho_1=2$ & $\rho_1=3$ & $\rho_1=\infty$\\ \hline
            $\rho_2=1$ & $67.50$ & $69.17$ & $67.29$ &  $68.54$\\
            $\rho_2=2$ & $68.33$ & $68.75$ & $\textbf{69.37}$ &  $67.61$\\
            $\rho_2=3$ & $67.08$ & $65.83$ & $67.50$ &  $67.92$\\
            $\rho_2=\infty$ & $67.08$ & $66.77$ & $67.67$ & $67.37$\\
            \hline
        \end{tabular} 
    }
\caption{\textit{Effect of $\rho_1$ and $\rho_2$ on sg-HeTRoM in outlier task setting. $\rho_1=\infty$ means no lower bound for task selection, and $\rho_2=\infty$ means no upper bound for task selection. }}\label{table:rho}
\end{table}

\section{Conclusion}
Meta-learning provides a solution to handle few-shot scenarios when dealing with many tasks. Existing optimization-based meta-learning methods operate on the premise that all tasks are equally important. However, this assumption is often challenged in real-world scenarios where tasks are inherently heterogeneous, characterized by varying levels of difficulty, noisy training samples, or a significant divergence from the majority of tasks. In response to these challenges, we have developed a novel meta-learning approach that proficiently manages these heterogeneous tasks using rank-based task-level learning objectives. A key advantage of our method is its ability to handle outlier tasks effectively, and it prevents an over-representation of similar tasks in influencing the meta-learner. Our method also employs an efficient iterative optimization algorithm grounded in bi-level optimization. Our experimental evaluations demonstrate the flexibility of our approach, as it enables users to adapt to a variety of task settings and improves the overall performance of the meta-learner. 

We plan to enhance our method in several ways in the future. Firstly, we aim to expand it to online meta-learning, where the model adjusts and learns in a real-time, sequential manner with the incoming data flow. Secondly, we intend to investigate its theoretical guarantees, including optimization error and generalization bound. Thirdly, we will test the effectiveness of our methods for other applications where meta-learning is helpful, such as object detection, speech recognition, and NLP.


\bibliographystyle{IEEEtran}

\appendix
\numberwithin{equation}{section}
\numberwithin{theorem}{section}
\numberwithin{figure}{section}
\numberwithin{table}{section}
\renewcommand{\thesection}{{\Alph{section}}}
\renewcommand{\thesubsection}{\Alph{section}.\arabic{subsection}}
\renewcommand{\thesubsubsection}{\Roman{section}.\arabic{subsection}.\arabic{subsubsection}}

\newtheorem{lemma}{Lemma}
\newtheorem{proof}{Proof}

\def\p{\mathbf{p}}
\def\v{\mathbf{v}}
\def\u{\mathbf{u}}

\setcounter{secnumdepth}{2}

{\bf \huge  Appendix}
\setcounter{section}{0}
\section{Settings for the Experiments in Section \ref{sec:sensitive-motivation}} \label{appendix:settings}

\noindent
\textbf{Influence of Majority Tasks.} In this setting, we use CNN4 as a backbone and train 2 meta-models on the Mini-ImageNet dataset with a mini-batch of size 16. During the meta-training phase, the two meta-models are updated iteratively using the entire mini-batch of tasks (traditional learning) and tasks with the top-8 highest losses (hard task mining) in each iteration.  After 2000 iterations of training, we perform fast adaptation on all the training tasks and examine the distribution of task losses on the query set. Additionally, the test accuracy of both meta-models on the same test dataset is presented in Figure \ref{fig:sensitivity}(a).

\noindent
\textbf{Sensitive to Noisy Tasks.} In this setting, we utilize two distinct datasets: clean dataset and noisy dataset. The clean dataset consists of 20,000 tasks generated from the original Mini-ImageNet dataset, while the noisy dataset involves corrupting 60\% of the tasks by randomly flipping all labels within each task. We train two separate meta-models on these datasets and subsequently perform fast adaptation on all tasks within the noisy dataset. Next, we separate the task losses of clean tasks (40\% of all tasks) and noisy tasks (60\% of all tasks) and show their respective distribution in Figure \ref{fig:sensitivity} (b). Note that the test accuracy shown in the figure corresponds to the performance evaluation conducted on the same clean test dataset.

\noindent
\textbf{Vulnerable to Outlier Task.} As described in \ref{experimental-settings}, the Tiered-ImageNet dataset is divided into categories, each containing a series of classes. In this setting, we choose two categories as target category and outlier category respectively. Two meta-models were trained, the first one using 16 tasks randomly sampled from the target category as clean tasks in each iteration of the mini-batch, and the second one using 15 tasks from the target category (indexed from 0 to 14) and 1 task from the outlier category (indexed as 15) in each iteration of the mini-batch. During training, we monitor the behavior of meta-training involving outlier tasks by recording the index of tasks with the highest task loss within the mini-batch in each iteration. Then we plot the frequency of each task index being the highest throughout the training process, which is shown in Figure \ref{fig:sensitivity} (c) together with the test accuracy of these two models on the same test dataset.

\section{Proof of Theorem \ref{theorem1}}\label{sec:proofs} 
To prove Theorem \ref{theorem1}, we first introduce the following lemma.
\begin{lemma} (\cite{hu2020learning})
    For a set of task losses $\mathbb{L} =\{\ell_1(\phi),...,\ell_n(\phi)\} $, and $\ell_{[i]}$ denotes the i-th largest loss after sorting the elements in $\mathbb{L}$, we have 
    \begin{align}
        \sum_{i=1}^k\ell_{[i]}(\phi)=min_{\lambda \in \mathbb{R}}\{k\lambda + \sum_{i=1}^n[\ell_i(\phi) - \lambda]_+\}. \nonumber
    \end{align}
    Furthermore, $\ell_{[k]}(\phi) \in \arg\min_{\lambda \in \mathbb{R}} \{ k\lambda + \sum_{i=1}^n[\ell_i(\phi) - \lambda]_+ \}$.
    \label{lemma1}
\end{lemma}
\begin{proof}
    We know $\sum_{i=1}^k\ell_{[i]}(\phi)$ is the solution of
    \begin{align}
        \max_\textbf{p} \textbf{p}^\top\mathbb{L}, \ \ \text{s.t.} \textbf{p}^\top \textbf{1} = k, \textbf{0} \leq \textbf{p} \leq \textbf{1}. \nonumber
    \end{align}
    We apply Lagrangian to this equation and get 
    \begin{align}
        L=-\textbf{p}^\top \mathbb{L} - \textbf{v}^\top \textbf{p} + \textbf{u}^\top (\textbf{p}-1) + \lambda(\textbf{p}^\top \textbf{1} - k), \nonumber
    \end{align}
    where $\textbf{u}\geq\textbf{0}$, $\textbf{v}\geq\textbf{0}$, $\lambda \in \mathbb{R}$ are Lagrangian multipliers. Taking its derivative w.r.t $\textbf{p}$ and set it to $0$, we have $v = \textbf{u} - \mathbb{L} + \lambda \textbf{1}$. Substituting it back to Lagrangian, we have
    \begin{align}
        \min_{\textbf{u}, \lambda} \textbf{u}^\top \textbf{1} + k \lambda, \text{s.t.} \textbf{u} \geq \textbf{0}, \textbf{u} + \lambda \textbf{1} - \mathbb{L} \geq \textbf{0}. \nonumber
    \end{align}
    This means
    \begin{align}
        \sum_{i=1}^k\ell_{[i]}(\phi) = \min_\lambda \bigg\{ k\lambda + \sum_{i=1}^n[\ell_i(\phi) - \lambda]_+ \bigg\}. 
    \label{eq:lemma1}
    \end{align}
    Furthermore, we can see that $\lambda=\ell_{[i]}(\phi)$ is always one optimal solution for Eq.\ref{eq:lemma1}. So
    \begin{align}
        \ell_{[k]}(\phi) \in \arg\min_\lambda \bigg\{ k\lambda + \sum_{i=1}^n[\ell_i(\phi) - \lambda]_+ \bigg\}.\nonumber
    \end{align}
\end{proof}

\begin{theorem}(Theorem \ref{theorem1} restated)
    Denote $[a]_+=\max\{0,a\}$ as the hinge function. Equation (\ref{eq:aorr_meta}) is equivalent to
\begin{align}
    \min_{\phi, \lambda} \max_{\hat{\lambda}}& \hat{\mathcal{L}}_{\mathcal{D}}(\phi,\bw^*)= \frac{1}{k-m}\sum_{i=1}^n\hat{\mathcal{L}}(\phi, w_i^*,\lambda,\hat{\lambda})\nonumber\\
    &:=\Big[[\ell_i(\phi)-\lambda]_+-[\ell_i(\phi)-\hat{\lambda}]_++\frac{k}{n}\lambda-\frac{m}{n}\hat{\lambda}\Big]\nonumber\\
    \emph{s.t.} \ \ \bw^* &= \arg\min_{\bw}\mathcal{L}_{\mathcal{S}}(\phi, \bw):=\frac{1}{n}\sum_{i=1}^n\mathcal{L}_{\mathcal{S}_i}(\phi, w_i),
    \end{align}
If the optimal of $\phi$ and $\bw^*$ are achieved, $(\ell_{[k]}(\phi),\ell_{[m]}(\phi))$ can be the optimum solutions of $(\lambda,\hat{\lambda})$.
\end{theorem}

\begin{figure*}[ht]
 \centering
 \includegraphics[width=\linewidth]{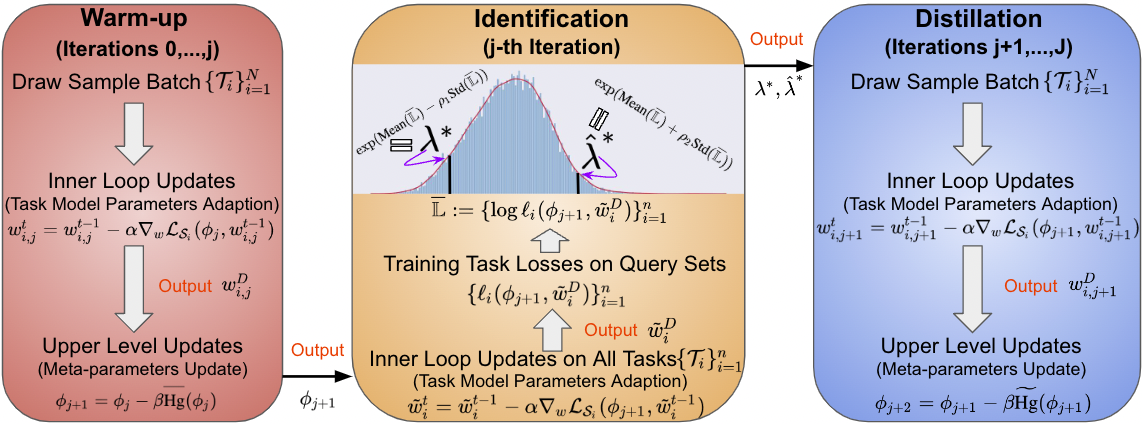}
 \caption{Illustration of statistic-guided learning for solving HeTRoM.}
 \label{fig:Statistic-guided-Learning}
\end{figure*}

\begin{proof}
    From Lemma \ref{lemma1} and upper-level formulation in (\ref{eq:aorr_meta}), we have
    \begin{equation}
        \begin{aligned}
            \min_{\phi} \mathcal{L}_{\mathcal{D}}&(\phi, \bw^*) =\frac{1}{k-m}\min_{\phi}\sum_{i=m+1}^k \ell_{[i]}(\phi) \\ &=\frac{1}{k-m} \min_{\phi} \bigg[\sum_{i=1}^k \ell_{[i]}(\phi) - \sum_{i=1}^m \ell_{[i]}(\phi)\bigg] \\
            & = \frac{1}{k-m} \min_{\phi} \bigg[\min_{\lambda \in \mathbb{R}}\bigg\{k\lambda + \sum_{i=1}^n[\ell_i(\phi) - \lambda]_+\bigg\} \\ & \ \ \ \ - \min_{\hat\lambda \in \mathbb{R}}\bigg\{m\hat\lambda + \sum_{i=1}^n[\ell_i(\phi) - \hat\lambda]_+\bigg\}\bigg] \\
            &= \frac{1}{k-m} \min_{\phi, \lambda} \max_{\hat\lambda}\sum_{i=1}^n\bigg[ [\ell_i(\phi) - \lambda]_+ - [\ell_i(\phi) - \hat\lambda]_+ \\
            & \ \ \ \ \ \ \ + \frac{k}{n}\lambda - \frac{m}{n}\hat\lambda \bigg]. \nonumber
        \end{aligned}
    \end{equation}
    According to Lemma \ref{lemma1}, if the optimal parameter $\phi^*$ is achieved, we have $\lambda = \ell_{[k]}(\phi^*)$ and $\hat\lambda = \ell_{[m]}(\phi^*)$. Since $k > m$, it is obvious that $\hat{\lambda} > \lambda$.
\end{proof} 

\section{The Algorithm of sg-HeTRoM} \label{sec:alg2}

In this section we provide comprehensive illustration of sg-HeTRoM method mentioned in \ref{sec:statistic-guided-learning}, as in Figure \ref{fig:Statistic-guided-Learning} and Algorithm \ref{alg:2}.

\begin{algorithm*}[t]
	\renewcommand{\algorithmicrequire}{\textbf{Require:}}
	\caption{Optimizing procedure of sg-HeTRoM}
	\label{alg:2}
	\begin{algorithmic}[1]
 \STATE {\bfseries Input:} Warmup iteration number $P$; total iteration number $J$; inner-loop iteration number $D$; batch size $N$; stepsize $\alpha, \beta, \gamma $, initial meta-parameter $\phi_0$; initial task-specific parameter $w_0$;  warmup flag $flag=1$; hyperparameters $\rho_1$, $\rho_2$.
            \FOR{$j=0,1,2,...,J$}
            \STATE{Sample a batch of tasks $\{\mathcal{T}_i\}_{i=1}^N$}
            \FOR{each task $\mathcal{T}_i$ in $\{\mathcal{T}_i\}_{i=1}^N$}
                \STATE{Initialize task-specific parameter as $w_{i,j}^{0} = w_0$}
                \FOR{$t=1,2,...,D$}
                \STATE{Update $w_{i, j}^t = w_{i, j}^{t-1}-\alpha \nabla_{w} \mathcal{L}_{\mathcal{S}_i}(\phi_j,w_{i, j}^{t-1}) $}
                \ENDFOR
            \ENDFOR

            \IF{$flag == 1$}
            \STATE{
                Calculate   hypergradient of $L_{\mathcal{D}}(\phi, \bw^*)$   w.r.t. $\phi$: $\overline{\text{Hg}}(\phi_j)=\frac{1}{N}\frac{\partial \sum_{i=1}^N \ell_i(\phi_j, w^D_{i,j})}{\partial  \phi_j}$
            }
            \STATE{
                Update meta-parameter $\phi_{j+1}=\phi_j - \beta\overline{\text{Hg}}(\phi_j)$ 
            }
            \IF{$j==P$}
                \FOR{each task $\mathcal{T}_i$ in the training set of all tasks $\{{\mathcal{T}}_i\}_{i=1}^n$}
                    \STATE{Initialize task-specific parameter as $\Tilde{w}_i^0 = w_0$}
                    \FOR{$t=1,2,...,D$}
                    \STATE{Update $\Tilde{w}_i^t = \Tilde{w}_i^{t-1}-\alpha \nabla_{w} \mathcal{L}_{\mathcal{S}_i}(\phi_{j+1},\Tilde{w}_i^{t-1}) $}
                    \ENDFOR
                \ENDFOR
                \STATE{Calculate log of all training task losses on their query set: $\overline{\mathbb{L}}=\{\log \ell_i(\phi_{j+1},\Tilde{w}_i^{D})\}_{i=1}^n$}
                \STATE{ Calculate+.
                $\begin{pmatrix}
                \lambda^*\\ 
                \hat{\lambda}^*
                \end{pmatrix} \leftarrow 
                \begin{pmatrix}
                    \exp(\text{Mean}(\overline{\mathbb{L}}) - \rho_1\text{Std}(\overline{\mathbb{L}}))\\
                    \exp(\text{Mean}(\overline{\mathbb{L}}) + \rho_2\text{Std}(\overline{\mathbb{L}}))
                \end{pmatrix}
                $
                }
                \STATE{$flag \leftarrow 0 $}
            \ENDIF

            \ELSE
		        \STATE{Update meta-parameter  $\phi_{j+1}=\phi_j - \beta\widetilde{\text{Hg}}(\phi_j)$ based on Eq.~\eqref{static-gradient}}
            \ENDIF
            \ENDFOR
        \end{algorithmic}  
\end{algorithm*}

\section{Implementation Details}\label{sec:add-exp-details}

\noindent
\textbf{Network Architecture.} In this study, we employ two distinct network architectures, namely 4-layer CNN (CNN4) and ResNet12. CNN4 comprises four convolutional layers with a feature dimension of 64, each followed by a ReLU activation layer and max pooling. On the other hand, ResNet12 consists of four residual blocks, each containing three convolutional layers, along with batch normalization and ReLU activation.

\noindent
\textbf{Training Parameters.} In the experimental section, we employ different hyperparameters for different network architectures.  For both architectures, we fix the total number of iterations to $J=2000$, and utilize Adam optimizer with a batch size of 16 throughout the training process. For CNN4, we set the inner-loop iteration number $D=10$, and learning rates are configured as $\alpha=0.05$, $\beta=\gamma=0.002$. On the other hand, the inner-loop iteration number of ResNet12 $D=3$, and the learning rates are $\alpha=0.5$, $\beta=\gamma=0.003$.

\noindent
\textbf{Implementation Details.} We conduct our experiments in a few-shot classification setting. Two types of neural network backbones are used for meta-training: a 4-layer CNN (CNN4) and ResNet12. Our proposed method is trained using a mini-batch ADAM algorithm with a batch size of $16$. The total iterations $J$ are set to $2000$, with inner-loop iterations $D$ at $10$. The initial learning rates are $\alpha=0.05$, $\beta=\gamma=0.002$. All our experiments are conducted using the PyTorch platform with two NVIDIA RTX A6000 GPUs. We construct our loss function based on the cross-entropy loss for the tasks. In line with \cite{hu2020learning}, we conduct a grid search to select the values of $k$ and $m$. We select $\rho_1$ and $\rho_2$ from the set $\{1,2,3\}$. The results we present, including a $95$\% confidence interval, are based on five random runs. Unless stated otherwise, all experiments are conducted based on Algorithm \ref{alg:1}. The detail of generating heterogeneous tasks can be found in Appendix \ref{sec:add-exp-details}.

\noindent
\textbf{Learning Rate Strategy for Training ResNet12 Model.} Throughout the meta-training process of the ResNet12 model, the initial individual task losses are high but undergo a significant decrease as training progresses. In order to ensure that our thresholds, $\lambda$ and $\hat{\lambda}$, are able to adapt to the changing scale of task losses, we introduce a scaling factor for the learning rate $\gamma$, denoted as $\delta$. This scaling factor starts with a large value at the beginning of meta-training and linearly decreases over time, ultimately reaching 0 at the end of training. This procedure is denoted as $\gamma_i = i\delta \gamma_0 /J$, where $\gamma_i$ denotes the learning rate of $\lambda$ and $\hat\lambda$ at $i$-th iteration, and $\gamma_0$ is the base learning rate. In our experiments, we set $\gamma_0=\beta=0.003$ and the scaling factor $\delta=1000$.

\noindent
\textbf{The Generation of Clean Tasks.} We designate the original Mini-ImageNet and Tiered-Imagenet datasets as clean datasets, from which we extract tasks for conducting experiments in the clean task setting. Our experiments focus on two few-shot classification tasks: 5-ways 5-shots, and 5-ways 1-shot. In both tasks, the ``way" number refers to the number of classes in the task and the ``shot" number indicates the number of examples sampled from each class.

\noindent
\textbf{The Generation of Noisy Tasks.} we consider three generation settings for noisy tasks: fixed, 1-step random, and 2-step random. (a) In the \textit{fixed} setting, we assign a specific noise probability for each of the tasks in a batch and then change the label of each sample to another randomly chosen label in each task according to the preset probability. For example, consider a batch of 16 tasks, where the first five tasks have no noise, the next five tasks have 10\% noise each, the following three tasks have 30\% noise each, the 14th and 15th tasks have 50\% noise each, and the 16th task has 70\% noise in each. {This setting can help us track the noise in each batch to show the process of eliminating noisy tasks from training.} (b) In \textit{1-step random} setting, for each task, we randomly choose its noise probability from $\{0\%, 10\%, 20\%, 30\%, 40\%\}$. (c) In \textit{2-step random}, following \cite{yao2021meta}, we first randomly select 60\% of tasks from all training tasks and set their noise probability to 80\%.

\noindent
\textbf{The Generation of Outlier Tasks.} (a) For the Tiered-ImageNet dataset, we consider the preset categories in Tiered-ImageNet as different domains and randomly select two categories: one is designated as the target category, and the other as the outlier category. We split all classes in each category into separate training and test parts with a 7:3 ratio. During the training phase, we randomly select a portion (\ie, $10$\%, $30$\%) of tasks as outlier tasks, whose samples are extracted from the training part of the outlier category. Samples from other tasks are taken from the training part of the target category. During the test phase, all tasks are sampled from the test part of the target category. (b) For Mini-ImageNet dataset, we cross-reference the classes in Mini-ImageNet with the classes in Tiered-ImageNet to determine the category to which each class in Mini-ImageNet belongs. This allows us to identify the categories included in Mini-ImageNet, and the remaining categories in Tiered-ImageNet are considered outlier categories. In training phase, we sample a portion (\ie, 10\%, 30\%) of tasks from outlier categories as outlier tasks, while the remaining tasks are sampled from Mini-ImageNet. During test phase, we exclusively evaluate the model's performance on Mini-ImageNet.

\section{Additional Experimental Results}\label{sec:add-exp-results}

\begin{table}[t]
\renewcommand\arraystretch{1.2}
\centering
    \scalebox{1}{
        \begin{tabular}{c|ccc|c}
            \hline
            Methods                & MAML   & ATS   & Eigen-Reptile & Ours \\ \hline
            Time & 1.22s   & 2.68s & 3.59s   & 1.36s     \\ \hline
        \end{tabular}
    }
\caption{\textit{Average iteration time of our method and baseline methods.}}\label{table:time}
\end{table}


\noindent
\textbf{Time complexity analysis.}  The time complexity of our method is comparable to MAML~\cite{finn2017model}, requiring only the maintenance of two extra scalar variables thus ensuring efficiency relative to other SOTA methods, which typically require additional computations (\eg, weight calculation in \cite{yao2021meta}). In Table \ref{table:time} we display the average iteration time of our method and baseline methods, showing that our method can achieve better efficiency.


\begin{table*}[t]
\renewcommand\arraystretch{1.2}
\center
\scalebox{0.8}{
\begin{tabular}{c|c|cc|ccc|ccc}
\hline
\multirow{2}{*}{Dataset}            & \multirow{2}{*}{Method} & \multicolumn{2}{c|}{Clean}      & \multicolumn{3}{c|}{Noisy (5-ways 5-shots)}                        & \multicolumn{3}{c}{Outlier (3-ways 5-shots)}                      \\
                                 &                         & 5w5s           & 5w1s          & Fixed          & 1-step         & 2-step         & ratio=0        & ratio=0.1      & ratio=0.3      \\ \hline
\multirow{3}{*}{Mini-ImageNet}   & MAML                                        & 48.72          & 42.25         & 41.00             & 41.85          & 38.95          & 65.83          & 65.33          & 64.58          \\ \cline{2-10}
                                 & HeTRoM(Ours)                                & \colorbox{llgray}{\textbf{59.90}}  & \colorbox{llgray}{46.25}         & \colorbox{llgray}{\textbf{59.45}} & \colorbox{llgray}{\textbf{59.15}} &\colorbox{llgray} {\textbf{58.25}} & \colorbox{llgray}{75.42}          & \colorbox{llgray}{72.50}           & \colorbox{llgray}{71.25}          \\
                                 & sg-HeTRoM(Ours)                             & \colorbox{llgray}{58.94}          & \colorbox{llgray}{\textbf{46.50}} & \colorbox{llgray}{57.60}           & \colorbox{llgray}{57.25}          & \colorbox{llgray}{56.75}          & \colorbox{llgray}{\textbf{79.41}} & \colorbox{llgray}{\textbf{78.75}} & \colorbox{llgray}{\textbf{77.08}} \\ \hline
\multirow{3}{*}{Tiered-ImageNet} & MAML                                        & 56.10           & 46.25         & 52.85          & 52.60           & 47.85          & 59.41          & 58.50           & 57.40           \\ \cline{2-10}
                                 & HeTRoM(Ours)                                & \colorbox{llgray}{\textbf{62.75}} & \colorbox{llgray}{\textbf{49.50}} & \colorbox{llgray}{\textbf{62.10}}  & \colorbox{llgray}{\textbf{61.65}} & \colorbox{llgray}{\textbf{61.65}} & \colorbox{llgray}{59.75}          & \colorbox{llgray}{59.50}           & \colorbox{llgray}{59.42}          \\
                                 & sg-HeTRoM(Ours)                             & \colorbox{llgray}{61.75}          & \colorbox{llgray}{47.25}        & \colorbox{llgray}{60.89}          & \colorbox{llgray}{61.25}          & \colorbox{llgray}{60.10}           & \colorbox{llgray}{\textbf{64.79}} & \colorbox{llgray}{\textbf{62.58}} & \colorbox{llgray}{\textbf{62.17}}  \\ \hline
\end{tabular}
}
\caption{\label{tab:res-100-tasks} \textit{Test accuracy (\%) on Mini-ImageNet and Tiered-ImageNet using CNN4 architecture with 100 training tasks.
The best results are shown in \textbf{Bold} for each setting. \colorbox{llgray}{Gray} indicates our methods outperform baseline methods in the same setting.}} 
\end{table*}

\begin{table*}[ht]
\centering
\scalebox{0.8}{
\begin{tabular}{c|c|cc|cc}
\hline
\multirow{2}{*}{Model}           & \multirow{2}{*}{Method} & \multicolumn{2}{c|}{Clean}         & \multicolumn{2}{c}{60\% noisy samples} \\ \cline{3-6} 
                                 &                         & \multicolumn{1}{c|}{5w5s}  & 5w1s  & \multicolumn{1}{c|}{5w5s}    & 5w1s    \\ \hline
\multirow{2}{*}{ResNet18}        & TIM w/o HeTRoM          & \multicolumn{1}{c|}{ $79.11 \pm 0.14$} & $59.53 \pm 0.21$ & \multicolumn{1}{c|}{$73.56 \pm 0.21$}   & $\textbf{57.21} \pm \textbf{0.21}$   \\
                                 & TIM w HeTRoM            & \multicolumn{1}{c|}{$\textbf{80.16} \pm \textbf{0.14}$} & $\textbf{61.03} \pm \textbf{0.20}$ & \multicolumn{1}{c|}{$\textbf{74.48} \pm \textbf{0.16}$}   & $57.20 \pm 0.20$   \\ \hline
\multirow{2}{*}{WideResNet28-10} & TIM w/o HeTRoM          & \multicolumn{1}{c|}{$\textbf{82.81} \pm \textbf{0.13}$} & $62.59 \pm 0.20$ & \multicolumn{1}{c|}{$74.02 \pm 0.16$}   & $57.37 \pm 0.21$   \\
                                 & TIM w HeTRoM            & \multicolumn{1}{c|}{$82.68 \pm 0.14$} & $\textbf{63.35} \pm \textbf{0.20}$ & \multicolumn{1}{c|}{$\textbf{74.28} \pm \textbf{0.16}$}   & $\textbf{57.49} \pm \textbf{0.21}$   \\ \hline
\end{tabular}
}
\caption{ Test accuracy on MiniImageNet compared with TIM.}
\label{exp:tim}
\end{table*}

\noindent
\textbf{Meta-learning with Extreme Few Tasks.} In real-world scenarios, users may face challenges accessing a substantial number of tasks for training the meta-model.  We conducted experiments under the constraint of a limited task set (specifically, 100 tasks in this setting) during training, and the outcomes are presented in Table \ref{tab:res-100-tasks}. The results reveal a noteworthy performance decline for MAML compared to the experiments detailed in Table \ref{table1}, rendering it impractical in such constrained settings. In contrast, our method demonstrates appropriate performance, with particularly noteworthy efficacy in noisy settings.

\noindent
\textbf{Incorporating with SOTA few-shot learning method.} 
In addition to SOTA meta-learning techniques, our sg-HeTRoM method is readily adaptable to other few-shot learning approaches with various training paradigms. We apply our sg-HeTRoM method at the sample level to integrate with TIM~\cite{boudiaf2020transductive}, with the results presented in Table \ref{exp:tim}. To be specific, we apply the same model settings as in \cite{boudiaf2020transductive}, and use the clean and 2-step random noisy setting as in Section \ref{experimental-settings}. In implementing our sg-HeTRoM approach at the sample level, we continue to adhere to the Warm-up, Identification, and Distillation framework specified in Section \ref{sec:statistic-guided-learning}, with the adaptation that the Identification phase is executed at the sample level. The results demonstrate that incorporating our method allows TIM to attain comparable or even superior performance across different model architectures in both clean and noisy environments, underscoring the effectiveness of our approach when combined with SOTA methods.

\end{document}